\documentclass[sigconf]{acmart}

\AtBeginDocument{%
  }

\setcopyright{acmcopyright}
\copyrightyear{2022}
\acmYear{2022}

\acmConference[Arxiv Preprint]{2022}{July}{2022 }
\acmPrice{15.00}

\usepackage[ruled,vlined]{algorithm2e}
\usepackage{subfigure}
\usepackage{multirow}
\usepackage{natbib}

\def\bA{{\bf A}}

\def\bH{{\bf H}}
\def\bI{{\bf I}}

\def\bZ{{\bf{Z}}}

\def\ba{{\bf a}}

\def\bc{{\bf c}}

\def\bg{{\bf g}}
\def\bh{{\bf h}}

\def\bq{{\bf q}}

\def\bv{{\bf v}}

\def\bx{{\bf x}}
\def\by{{\bf y}}
\def\bz{{\bf z}}

\begin{document}

\title{Quantized Adaptive Subgradient Algorithms and Their Applications}

\author{Ke Xu, Jianqiao Wangni, Yifan Zhang, Deheng Ye, Jiaxiang Wu  and  Peilin Zhao 
}
\renewcommand{\shortauthors}{Ke Xu, Jianqiao Wangni, Yifan Zhang, Deheng Ye, Jiaxiang Wu  and  Peilin Zhao}

\begin{abstract}
Data explosion and an increase in model size drive the remarkable advances in large-scale machine learning, but also make model training time-consuming and model storage difficult. To address the above issues in the distributed model training setting which has high computation efficiency and less device limitation, there are still two main difficulties. On one hand, the communication costs for exchanging information, e.g., stochastic gradients among different workers, is a key bottleneck for distributed training efficiency. On the other hand, less parameter model is easy for storage and communication, but the risk of damaging the model performance. To balance the communication costs, model capacity and model performance simultaneously, we propose quantized composite mirror descent adaptive subgradient (QCMD adagrad) and quantized regularized dual average adaptive subgradient (QRDA adagrad) for distributed training. To be specific, we explore the combination of gradient quantization and sparse model to reduce the communication cost per iteration in distributed training. A quantized gradient-based adaptive learning rate matrix is constructed to achieve a balance between communication costs, accuracy, and model sparsity. Moreover, we theoretically find that a large quantization error brings in extra noise, which influences the convergence and sparsity of the model. Therefore, a threshold quantization strategy with a relatively small error is adopted in QCMD adagrad and QRDA adagrad to improve the signal-to-noise ratio and preserve the sparsity of the model. Both theoretical analyses and empirical results demonstrate the efficacy and efficiency of the proposed algorithms.
\end{abstract}





\maketitle

\section{Introduction}
The large model has recently been extremely successful in machine learning and data mining with the growth of data volume. A great number of complex deep neural networks \cite{zhao2018adaptive,zhang2018online,krizhevsky2012imagenet,zhang2019whole,cao2019multi} have been devised to solve real-world applications problem. However, many practical applications can only provide limited computing resources, i.e.  limited storage devices and unaccelerated hardware such as CPUs.
These constrain the model complexity and make model training extremely time-consuming under the huge amount of training datasets. 
To solve these issues, this paper explores to
accelerate model training and reduce model storage costs simultaneously for large-scale machine learning.

Distributed training offers a potential solution to solve the issue of long training time~\cite{balcan2016communication,chilimbi2014project,xing2015petuum,hsieh2017communication,zhang2017online}. 
The data parallelism \cite{li2017scaling, gupta2016model} is one of the most popular framework in distributed training. 
As shown in Fig.~\ref{parallel}, data parallelism framework has one or more computer workers connected via some communication networks, while multiple model replicas are trained in parallel on each worker. A global parameter server 
ensures the consistency among replicas of each worker  by collecting all gradients computed from different workers and then averaging them to update parameters. 
The goal is to optimize a global objective function formed by the average of a series of local loss functions derived from local computation on each worker. 

\begin{figure}[h]
    \centering
    \includegraphics[width=0.58\linewidth]{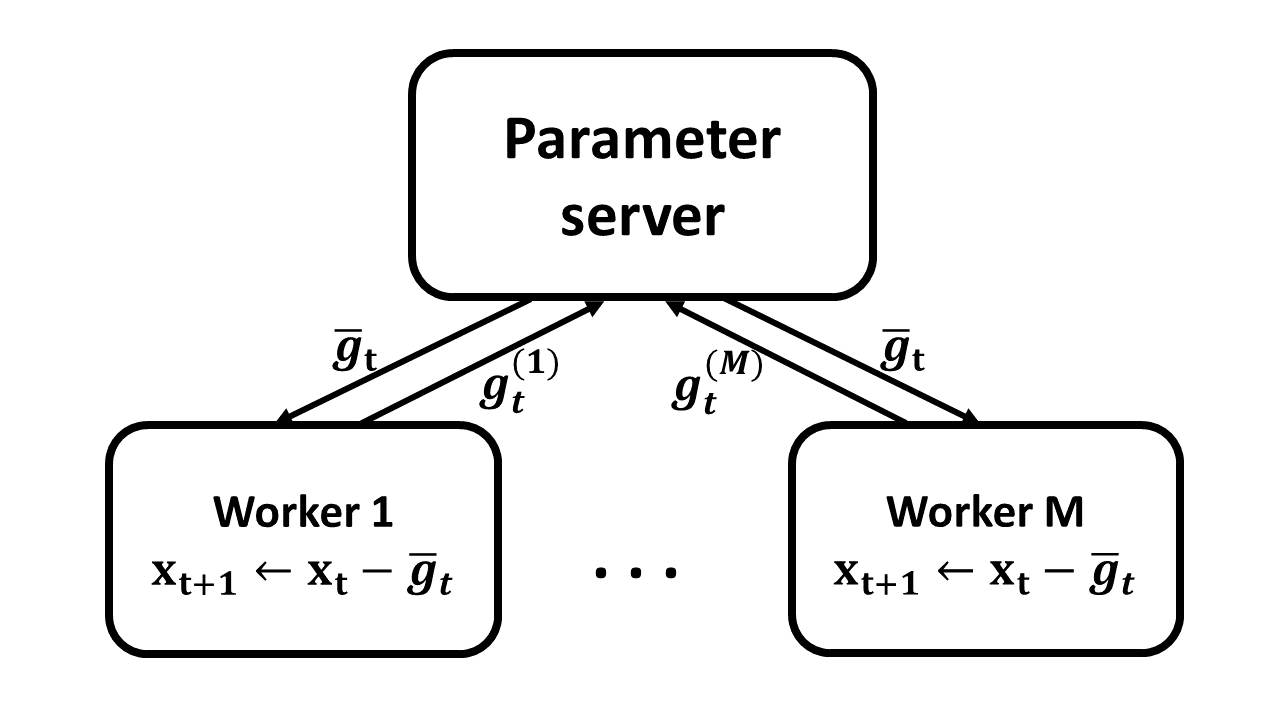}
    \caption{Data parallelism in a typical distributed training scheme, where $g_t^{(m)}$ denotes the local gradient computed by the $m^{th}$ worker, and $M$ denotes the total number of workers. The parameter server collects all local gradients from the workers. In addition, $\bar{g}_t$ is the averaged gradients, which is 
    synchronized by the parameter server and pulled by each worker to update the model.}
    \label{parallel}
\end{figure}

In the above distributed framework, time consumption is mainly caused by computation and communication.
While increasing the number of workers helps to reduce computation time, the communication overhead for exchanging training information, e.g., stochastic gradients among different workers, is a key bottleneck for training efficiency, especially for high delay communication network. Even worse, some slow workers can adversely affect the progress of fast workers, leading to a drastic slowdown of the overall convergence process \cite{eshraghi2020distributed}. Asynchronous communication \cite{huo2018asynchronous}, alleviate the negative effect of slow machines and decentralized algorithms \cite{lian2017can,xin2020decentralized} get rid of the dependency on high communication costs on the central node. But, the parameter variance among workers  may  deteriorate  model  accuracy\cite{ko2021aladdin}. In this paper, we concentrate on the data parallelism framework which belongs to synchronous centralized algorithms. There are many types of research focusing on how to save communication costs. 
Some studies \cite{bekkerman2011scaling, de2016efficient} propose to reduce the number of training rounds. 
 For example, one may use SVRG \cite{johnson2013accelerating,shah2016trading}  to periodically calculate an accurate gradient estimate  to reduce the variance introduced by stochastic sampling, which, however, is an operation with high calculation costs. 
 Other methods focus on reducing the precision of gradients. DoReFa-Net \cite{zhou2016dorefa} and QSGD \cite{alistarh2017qsgd} quantize gradients into fixed-point numbers, so that much fewer bits are needed to be transmitted. More aggressive quantization methods, such as 1-bit-sgd \cite{seide20141} and ternary gradient \cite{wen2017terngrad}, sacrifice a certain degree of expression power to achieve the purpose of reducing communication costs. \cite{tang2019doublesqueeze} studies the double squeeze of the gradient that not only the local gradient is compressed but also the synchronous gradient.

The size of the model is not only a determinant of memory usage but also an important factor of save the communication cost in distributed training.  
Although in data parallelism, we transmit the gradient instead of the model parameters. But for a sparse model , we can avoid transmitting unnecessary gradients corresponding to the parameters that always be 0.
The combination of the above bandwidth reduction methods and parallel stochastic gradient descent (PSGD) has been demonstrated to be effective for model training without model sparsity constraints. Therefore, inspired by some online algorithms such as adaptive composite mirror descent and adaptive regularized dual averaging  \cite{duchi2011adaptive}, which perform well in generating sparse models,
we design two corresponding communication efficient sparse model training algorithms for distributed framework, named quantized composite mirror descent adaptive subgradient \textbf{\textit{(QCMD adagrad)}} and quantized regularized dual averaging adaptive subgradient \textbf{\textit{(QRDA adagrad)}}.


To be specific, we define the distributed objective not only under a nonsmooth constraint to keep the model sparse, but also construct a proximal function base on a quantized gradient to achieve a balance between communication costs, accuracy, and model sparsity. The quantization is conducted not only for local gradients computed by all local workers but also for the aggregated gradient like the double squeeze \cite{tang2019doublesqueeze}.
We proof the convergence rate for distributed QCMD adagrad and QRDA adagrad is $O(\frac{1}{\sqrt{T}})$. Besides,
 our theoretical analysis shows that quantization introduces additional noise which  affects the model convergence and sparsity. Hence, we apply a threshold quantization method with small error to the gradient to reduce the influence of noise on model training. 

 
Our major contributions are summarized as follows:
\begin{itemize}
    \item We propose two communication efficient sparse model training algorithms for the distributed framework, namely  QCMD adagrad and QRDA adagrad.
    \item  We theoretically find that quantization noise affects the model convergence and sparsity, and thus we adopt a small error quantization method to alleviate noise influence. We provide theoretical results on the regret for the QCMD adagrad and QRDA adagrad. The convergence rate is $O(\frac{1}{\sqrt{T}})$.
    \item We apply QCMD adagrad and QRDA adagrad to both linear models and convolutional neural networks. Experimental results demonstrate the effectiveness and efficiency of the proposed methods in convex and non-convex problems.
\end{itemize}


\section{Related Work\label{sct_related_works}}
We detail related studies in three aspects as follows.

\textbf{Gradients sparsification.} Gradient sparsification \cite{seide20141} imposes sparsity onto gradients, where only a small fraction of elements of gradients are exchanged across workers based on their importance. Lin et al. \cite{lin2017deep} find most of the gradient exchange in distributed SGD are redundant, and propose a deep gradient compression method to cut the gradient size of ResNet-50 from 97MB to 0.35MB, and that of DeepSpeach from 488MB to 0.74MB.
Aji et al. \cite{aji2017sparse} sparsify gradients by removing the $R\%$ smallest gradients regarding absolute value. Simlarly, \cite{stich2018sparsified} propose sparsified SGD with top k-sparsification. Wangni et al.  \cite{wangni2017gradient} analyse how to achieve the optimal sparseness under certain variance constraints and the unbiasedness of the gradient vector is kept after sparsification. 

\textbf{Gradients quantization.}
Gradient quantization replaces the original gradient with a small number of fixed values. 
Quantized SGD (QSGD) \cite{alistarh2017qsgd} adjusts the number of bits of the exchanging gradients
to balance the bandwidth and accuracy.
More aggressive quantization methods, such as the binary representation \cite{seide20141}, cut each component of the gradient to its sign. 
TernGrad \cite{wen2017terngrad} uses three numerical levels \{-1, 0, 1\} and a scaler, e.g. maximum norm or $l_2$ norm of the gradient, to replace the 
full precision gradient. This aggressive method can be regarded as a variant of QSGD. To reduce the influence of noise introduced by  the aggressive quantization, Wu et al. \cite{wu2018error} utilize the accumulated quantization error to compensate the quantized gradient. Several applications, such as federated machine learning (ML) at the
wireless edge \cite{amiri2020machine}, benefit from the error compensation. In addition, \cite{magnusson2020maintaining} introduce a family of adaptive gradient quantization schemes which can enable linear convergence in any norm for gradient-descent-type algorithms.
\cite{alimisis2021communication} propose a quantized Newton’s method which is suitable for ill-conditioned but low-dimensional problems, as it reduces the communication complexity by a trade-off between the dependency on the dimension of the input features and its condition number. In \cite{alimisis2021communication}, lattice quantization \cite{davies2021new} that reduces the variance is adapted to quantize the covariance matrix.

\textbf{Stochastic optimization.} In the modern implementation of large-scale machine learning algorithms, stochastic gradient descent (SGD) is commonly used as the optimization method in distributed training frameworks because of its universality and high computational efficiency for each iteration. 
SGD intrinsically implies gradient noise, which helps to escape saddle points for non-convex problems, like neural networks \cite{jin2017escape, kleinberg2018alternative}.
However, when producing a sparse model, simply adding a subgradient of the $l_1$ penalty to the gradient of the loss does not essentially produce parameters that are exactly zero. More sophisticated approaches such as composite mirror descent  \cite{singer2009efficient,duchi2010composite} and regularized dual averaging \cite{xiao2010dual} do succeed in introducing sparsity, but the sparsity of model is limited. Their adaptive subgradient extensions (Adagrad) with $l_1$ regularization \cite{duchi2011adaptive} produce  even better accuracy vs. sparsity tradeoffs.
Compared with SGD which is very sensitive to the learning rate, Adagrad  \cite{duchi2011adaptive} dynamically incorporates knowledge of the geometry of the data curvature of the loss function
to adjust the learning rate of gradients. As a result, it requires no manual tuning of learning rate and appears robust to noisy gradient information and large-scale high-dimensional machine learning. 

\begin{figure*}[tp]
\vspace{-0.22in}
    \centering
    \includegraphics[width=\linewidth]{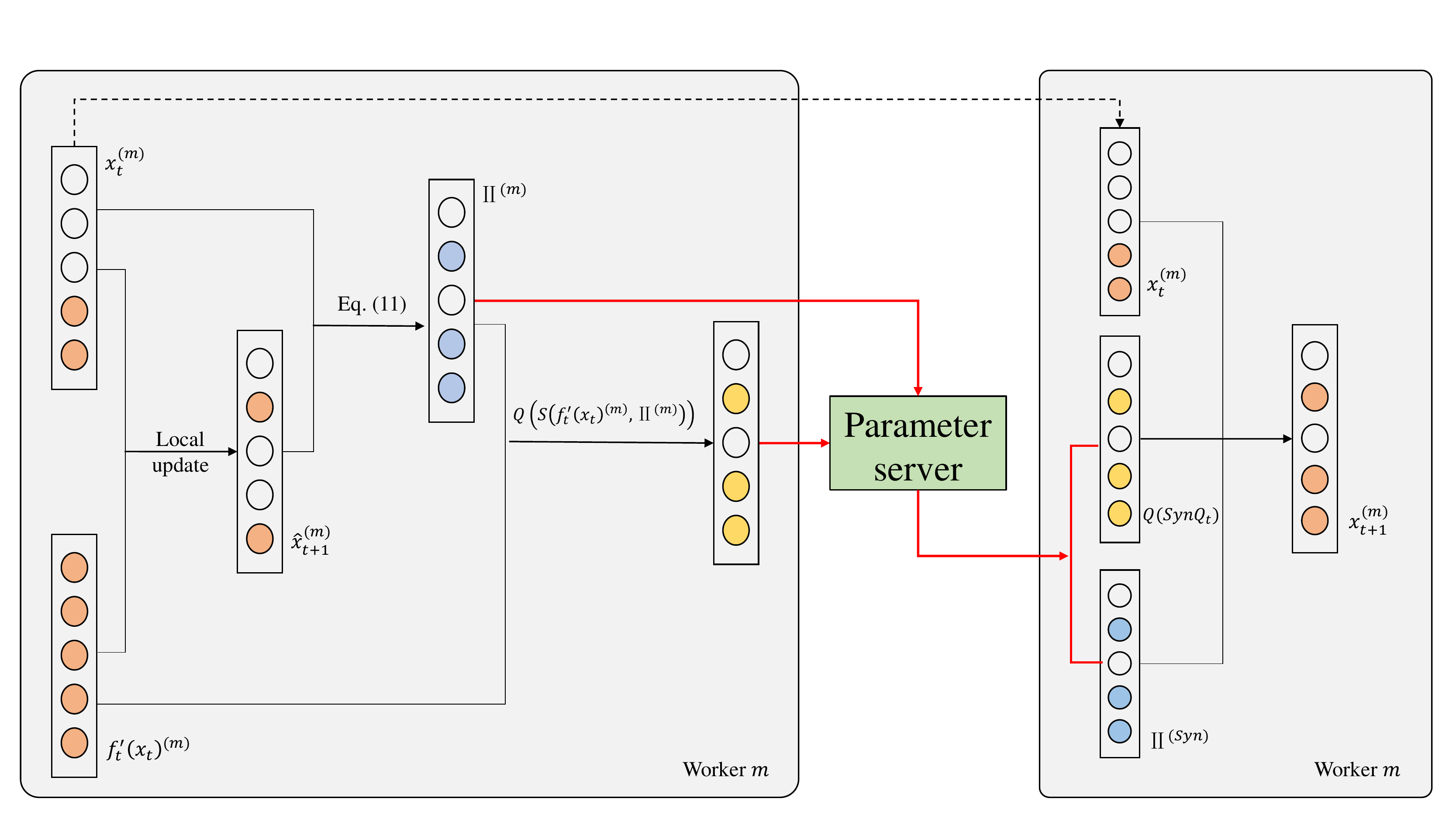}
\vspace{-0.22in}
    \caption{General framework of QCMD adagrad and QRDA adagrad. $\mathbb{I}$ is the indicator function defined in Eq.(\ref{ind}).  $\hat\bx_{t+1}^{(m)}$ is the local parameter that is updated based on  the  $f'_t(\bx_t)^{(m)}$.  $S(\cdot)$ is the selection function defined in Eq.(\ref{select}). Nodes with colors denote the non-zero elements. Quantization function $Q(\cdot)$ quantize the input. Both the indicators and the quantized selected gradients are needed to be sent to the server for the help of decoding. Finally, the parameter server synchronizes the indicators and quantized selected gradients from all $M$ workers and sent them to the local workers for the next step update.} 
    \label{indicator}
\end{figure*} 

\section{Problem Definition and our approach \label{sct_prob_def_appro}}
\textbf{Notations:}
Before the problem definition, we define the following notations.
Lower case bold letters, such as $\bv$, are used to denote vectors. Matrices are  capital case bold letters, like $\bA$. Scalars are lower case italic letters, such as $s$. We use [d] to denote the set $\{1,2,...,d\}$. For a vector $\bv\in\mathbb{R}^d$, the $l_q$ norm defined as $||\bv||_q=(\sum_{i\in[d]}|\bv_i|^q)^{\frac{1}{q}}$, and $||\bv||_\infty=\max_{i\in[d]}|\bv_i|$.
For a symmetric matrix $\bA \in \mathbb{R}^{d\times d}$ and any vector $\bx\in \mathbb{R}^d$, $\by\in \mathbb{R}^d$, 
the Bregman divergence associated with a strongly convex and differentiable function $\psi(\bx)=\frac{1}{2}\bx^\top \bA\bx$ is defined as $B_{\psi}(\bx,\by)=\frac{1}{2}(\bx-\by)^\top \bA(\bx-\by)$. The Mahalanobis norm $||\bx||_\psi=(\bx^\top \bA\bx)^{\frac{1}{2}}$ and it's associated dual norm $||\bx||_{\psi^*}=(\bx^\top \bA^{-1}\bx)^{\frac{1}{2}}$.
$<\bx,\by>$ is used to denote the inner product between $\bx$ and $\by$.
We also make frequent use of the following matrix, where $\bv_{1:t} = [\bv_1\cdot\cdot\cdot \bv_t]$ denote the matrix obtained by concatenating the vector sequence. We denote the $i^{th}$ row of $\bv_{1:t}$ by $\bv_{1:t,i}$, which amounts to the concatenation of the $i_{th}$ component of each vector.

\subsection{Problem Definition}
\textbf{Stochastic Optimization~\cite{Lee:2018:DQA:3219819.3220075}} is a popular approach in training large-scale machine learning models. We denote the objective function as $f(\bx)$, depending on the model parameter $\bx \in \chi$. The dataset could be split into $M$ workers, where each  has $N$ data samples. We denote $f_{m,n}$ as the loss function associated with the $n^{th}$ sample on the $m^{th}$ worker, and $f(\bx)$ is an average over all loss functions on all workers.
\begin{equation}
 f(\bx)=\frac{1}{M} \frac{1}{N} \sum_{m,n} f_{m,n} (\bx)
\end{equation}
When training on a single machine, 
the stochastic optimization, like SGD, randomly chooses a subset of samples $\bZ_t$ and calculates the gradient $f'_t(\bx_t,\bz_t)$, where $\bz_t\in\bZ_t$, to approximates the true gradient $ f'_t(\bx_t)$ at $t^{th}$ iteration by
\begin{equation}
  f'_t(\bx_t) = \frac{1}{|\bZ_t|}\sum_{\bz_t\in\bZ_t}(f'_t(\bx_t,\bz_t))
\end{equation}
Similarly, in a distributed setting, each worker calculates the stochastic gradient $f'_t(\bx_t)^{(m)}$ based on its local data $\bZ_t^{(m)}$. 
The parameter server averages the local gradients to get the synchronized gradient,
\begin{equation}
SynG_t(\bx_t)=\frac{1}{M} \sum_{m=1}^{M}  f'_t(\bx_t)^{(m)}
\end{equation}
As previous work on gradient sparsification and quantization suggested, we could represent $f'_t(\bx_t)^{(m)}$ in a low-precision way. Since extreme gradient sparsification can also be seen as gradient quantization, this paper focuses on reducing communication costs through gradient quantization. We denote $Q(\cdot)$ as the quantization function and then the synchronized quantized gradient can be obtained by
\begin{equation}
    SynQ_t=\frac{1}{M}\sum_{m=1}^{M}Q(f_t'(\bx_t)^{(m)})
\label{SynQ_t}
\end{equation}
The update formula for the synchronous distributed stochastic optimization with function $Q(\cdot)$ is 
\begin{equation}
 \bx_{t+1} = \bx_t -\eta \frac{1}{M} \sum_{m=1}^{M} Q(f_t'(\bx_t)^{(m)})
\end{equation}
where $\eta$ is the learning rate.


\subsection{Quantized CMD Adagrad and Quantized RDA Adagrad}

\textbf{General Framework}. In the data parallelism distributed training scheme, the number of the exchanged gradients is equal to the number of model learnable parameters. But for a sparse model, we can avoid transmitting unnecessary gradient information, i.e. gradient corresponding to parameters that are always 0 before and after model update. To this end, each worker utilizes their local full precision gradients to obtain the model parameters after the local update. 
With the assumption that the model parameters change smoothly and be sparsity, we can transmit a small amount of the necessary gradient to save the communication cost. Furthermore, this small amount of the local gradients is quantized and encoded. After the global parameter server receives and decodes the quantized gradients, synchronous gradients are calculated by averaging the local sparse quantized gradients. The synchronous sparse gradients are then quantized and sent to the workers for model parameter update. The general framework of QCMD adagrad and QRDA adagrad is shown in Fig. \ref{indicator}. 

\textbf{Generate Sparse Model}. Vanilla SGD is not particularly effective at producing a sparse model, even adding a subgradient of $l_1$ penalty to the gradient of the loss. Some other approaches such as proximal gradient descent (also named composite mirror descent) \cite{singer2009efficient,duchi2010composite} and regularized dual averaging \cite{xiao2010dual} introduce limited sparsity with $l_1$ regularization.  Their adaptive subgradient extensions, CMD adagrad and RDA adagrad,  \cite{duchi2011adaptive} produce  better accuracy vs. sparsity tradeoffs. 
  
  The original proximal gradient method \cite{duchi2010composite} employs an immediate trade-off among the current gradient $f_t'(\bx_t)$, the regularizer $\phi$ and proximal function $\psi$. The proximal function $\psi$ aims to keep $\bx$ staying close to $\bx_t$ and sometimes be simply set to $\psi(\bx)=\frac{1}{2}||\bx-\bx_t||^2_2$). It makes the model parameters meet the assumption of steady change. To achieve better bound of regret, the CMD adagrad adopts $\psi_t$ which varies with $t$.   The update  for CMD adagrad amounts to solving
\begin{equation}\label{CMD1}
\bx_{t+1}=\arg \min_{\bx\in\chi}{\eta f'_t(\bx_t)^\top}\bx+\eta \phi(\bx)+B_{\psi_t}(\bx,\bx_t)
\end{equation}
Similarly, the RDA adagrad encompasses a trade off among a gradient-dependent linear term, the regularizer $\phi$ and a strongly convex term $\psi_t$ for well conditioned predictions. The update for RDA adagrad is shown as:
\begin{equation}\label{RDA1}
    \begin{split}
        \bx_{t+1}=\arg\min_{\bx\in \chi}\{{\eta}<\frac{1}{t}\sum_{\tau=1}^tf'_\tau(\bx_\tau),\bx>+\eta\phi(\bx)+\frac{1}{t}\psi_t(\bx)\}
    \end{split}
\end{equation}
For both CMD adagrad and RDA adagrad, $\phi(\bx)=\lambda||\bx||_1$ encourages the sparsity of the model parameter. The Bregman divergence associated with $\psi_t(\bx)=\frac{1}{2}\bx^\top \bA_t\bx$ is defined as
\begin{align}
B_{\psi_t}(\bx,\bx_t)=\frac{1}{2}(\bx-\bx_t)^\top \bA_t(\bx-\bx_t)
\end{align}
$\bA_t$ is a diagonal matrix also named the adaptive learning rate matrix. Concretely, for some small fixed $\delta\geq 0$, the adaptive learning rate matrix is:
\begin{align}
\bA_t=\delta \bI+diag(\ba_t)
\end{align}
where $\ba_{t,d}=||f'(\bx)_{1:t,d}||_2$, $f'(\bx)_{1:t}=[f'(\bx)_{1:t-1},f'_t(\bx_t)]$ denote the matrix obtained by concatenating the gradient sequence.  $f'(\bx)_{1:t,d}$ is the $d^{th}$ row of this matrix and $diag(\cdot)$ converts a vector into a diagonal matrix.
$\bI$ is identity matrix.

To apply them to distributed learning, each worker needs to maintain a adaptive learning rate matrix locally. On the $m_{th}$ worker, the elements of the gradient $f'(\bx_t)^{(m)}$ are selected by a selected function:
\begin{align}
S(f'_t(\bx_t)^{(m)},\mathbb{I}^{(m)})=\{f'_{t,d}(\bx_t)^{(m)}|\mathbb{I}_d^{(m)}=1\}
\label{select}
\end{align}
where $\mathbb{I}^{(m)}$ is an indicator function that

\begin{align} \mathbb{I}_d=\left\{
\begin{aligned}
1, & \text{ } \bx_{t,d}\neq 0 \text{ or } \hat\bx_{t+1,d}\neq 0 \\
0, & \text{ otherwise}
\end{aligned}
\right.
\label{ind}
\end{align}
$\hat\bx_{t+1}^{(m)}$ is the local parameter that is updated based on  the local full precision gradient $f'_t(\bx_t)^{(m)}$ and local adaptive learning rate matrix $\bA_t^{(m)}$. Both the indicator and selected elements of the local gradients are sent to the global parameter server for decoding. The synchronous indicator is calculated by "bitwise or" between the $M$ indicators.

\textbf{Gradient quantization.} Besides reducing the transmission amount of the corresponding gradient by generating a sparse model, we also hope to further reduce the communication cost by quantifying the selected elements of the gradient. 
The quantization function is define as $Q(\cdot)$. Each worker computes $Q(S(f'_t(\bx_t)^{(m)},\mathbb{I}^{m}))$. The global parameter server first decodes the quantized gradient through $S^{-1}$ and computes the synchronized quantized gradient based on Eq.(\ref{SynQ_t}).
We furtherly quantize the synchronized gradient $Syn Q_t$ to save more communication costs by
 
\begin{equation}\label{Qt}
     \bq_t=Q(Syn Q_t)
\end{equation}
The server pushes the double quantized gradient $\bq_t$ to every worker. Specially, the adaptive learning rate is calculated by the worker according to $\bq_t$ sequence. So we construct a quantized gradient based adaptive learning rate matrix,
\begin{align}\label{Hq}
\bH_t=\delta \bI+diag(\bc_t)
\end{align}
where $\bc_{t,d}=||\bq_{1:t,d}||_2$, $\bq_{1:t}=[\bq_{1:t-1},\bq_t]$.  
$\bH_t$ dynamically incorporates knowledge of the geometry of the data and the curvature of the loss function based on quantized gradient, which makes QCMD adagrad and QRDA adagrad achieve a good balance between model sparsity, accuracy and communication cost. Therefore, the objective function for QCMD adagrad becomes
\begin{equation}\label{CMD2}
\bx_{t+1}=\arg \min_{x\in\chi}{\eta \bq_t^\top}\bx+\eta \phi(\bx)+B_{\psi_t}(\bx,\bx_t)
\end{equation}
The objective function for QRDA adagrad becomes
\begin{equation}\label{RDA2}
    \begin{split}
        \bx_{t+1}=\arg\min_{\bx\in \chi}\{{\eta}<\frac{1}{t}\sum_{\tau=1}^t\bq_\tau,\bx>+\eta\phi(\bx)+\frac{1}{t}\psi_t(\bx)\}
    \end{split}
\end{equation}
Solving Eq.~(\ref{CMD2}), we have the update rule for QCMD adagrad:
\begin{equation}\label{QCMD}
\bx_{t+1,i}=sign(\bx_{t,i}-\eta \bH_{t,ii}^{-1}\bq_{t,i})[|\bx_{t,i}-\eta \bH_{t,ii}^{-1}\bq_{t,i}|-\lambda\eta \bH_{t,i}^{-1}]_{+}
\end{equation}
The subscript $(\cdot)_+$ represents that we retain the value greater than 0. Solving Eq.~(\ref{RDA2}), the update rule for QRDA adagrad becomes:
\begin{equation}\label{QRDA}
\bx_{t+1,i}=sign(-\sum_{\tau=1}^t\bq_{\tau,i})t\eta \bH_{t,ii}^{-1}[|\frac{1}{t}\sum_{\tau=1}^t\bq_{\tau,i}|-\lambda]_{+}
\end{equation}

The overall QCMD adagrad and QRDA adagrad schemes are presented in Alg.~\ref{algorithm 1}.\\

\begin{algorithm}[b]
    \caption{ QCMD adagrad and QRDA adagrad.}\label{algorithm 1}        
    \textbf{for} $t=1$ \textbf{to} $T$\\
    \quad\textbf{for M worker:} $m=1,...,M$ \textbf{do in parallel}\\    
    \quad\quad Randomly chooses $\bZ_t^{(m)}$ from the local data.\\
    \quad\quad Compute the gradients $f_t'(\bx_t)^{(m)}$ under $\bZ_t^{(m)}$.\\
    \quad\quad Compute the $\hat{\bx}_{t+1}^{(m)}$ based on  $f'_t(\bx)^{(m)}$\\ 
    \quad\quad and $\hat{\bH}_{t}=\delta\bI+diag(\hat{\bc}_t)$, where\\
    \quad\quad\quad$\hat{\bc}_{t,d}=|| [\bq_{1:t-1},f_t'(\bx_t)^{(m)}] ||_2$.\\
    \quad\quad Compute the indicator function $\mathbb{I}^{m}$.\\
    \quad\quad Quantize the selected elements of the gradients with \\
    \quad\quad\quad $Q(S(f_t'(\bx_t)^{(m)},\mathbb{I}^{(m)}))$.\\
    \quad\quad Push $Q(S(f_t'(\bx_t)^{(m)},\mathbb{I}^{(m)}))$ and $\mathbb{I}^{m}$ to the server.    \\
    \quad\textbf{Server:}\\
        \quad\quad Decode $Q(S(f_t'(\bx_t)^{(m)},\mathbb{I}^{(m)}))$  to get  $Q(f_t'(\bx_t)^{(m)})$. \\
        \quad\quad Compute the synchronous indicator $\mathbb{I}^{Syn}$.\\
        \quad\quad Average the quantized gradients\\
        \quad\quad\quad                 $SynQ_t=\frac{1}{M}\sum_{m=1}^{M}Q(f_t'(\bx_t)^{(m)})$.\\
        \quad\quad Quantize the  aggregated gradient $Syn Q_t$\\
        \quad\quad\quad  $\bq_t=Q(Syn Q_t)$.\\
        \quad\quad Push $S(\bq_t,\mathbb{I}^{Syn})$ and $\mathbb{I}^{Syn}$ to every worker. \\
 
    \quad\textbf{for M worker:} $j=1,...,M$ \textbf{do in parallel}\\
    \quad\quad Pull $S(\bq_t,\mathbb{I}^{Syn})$ and $\mathbb{I}^{Syn}$ from the server.\\
    \quad\quad Decode $S(\bq_t,\mathbb{I}^{Syn})$  to get $\bq_t$. \\
    \quad\quad Calculate the adaptive learning rate\\ 
        \quad\quad\quad  $\bH_t=\delta \bI+diag(\bc_t)$,\\
    \quad\quad where $c_{t,d}=||\bq_{1:t,d}||_2$, $\bq_{1:t}=[\bq_{1:t-1},\bq_t]$.\\    
    \quad\quad QCMD adagrad: update $\bx_{t+1}$ based on  Eq.~(\ref{QCMD}).\\
    \quad\quad QRDA adagrad: update $\bx_{t+1}$ based on  Eq.~(\ref{QRDA}).\\
\end{algorithm}

\subsection{Quantization error}
In this section, we theoretically analyse that the quantization error introduced by the gradient quantization affects the convergence rate of regret for QCMD adagrad and QRDA adagrad. 
\begin{proposition}
Let the sequence ${\bx_t}$ be defined by the update (\ref{QCMD}), $SynQ_t$ defined by Eq.~(\ref{SynQ_t}), $\bq_t$ defined by Eq.~(\ref{Qt}), $\mathbb{E}[\bq_t]=f_t(\bx_t)$. The Mahalanobis norm $||\cdot||_{\psi_t}=\sqrt{<\cdot,\bH_t\cdot>}$ and $||\cdot||_{\psi^*_t}=\sqrt{<\cdot,\frac{1}{\bH_t}\cdot>}$ be the associated dual norm. $\bx^*$ is the optimal solution to $f(\bx)$, for any $\bx^*\in\chi$,
\begin{equation*}
\begin{split}
 &\quad\sum_{t=1}^{T}\mathbb{E}_\bq [f_t(\bx_{t+1})+\phi(\bx_{t+1})-f_t(\bx^*)-\phi(\bx^*)]\\
&\leq  \frac{1}{2}\sum_{t=1}^{T}\mathbb{E}_\bq||\bq_t||^2_{\psi^*_t}+\frac{1}{\eta}\mathbb{E}_\bq B_{\psi_1}(\bx^*,\bx_1)\\
&\quad+\frac{1}{\eta}\sum_{t=1}^{T-1}\mathbb{E}_\bq [B_{\psi_{t+1}}(\bx^*,\bx_{t+1})-B_{\psi_t}(\bx^*,\bx_{t+1})]\\
\end{split}    
\end{equation*}
\end{proposition}
\begin{proof}
See Appendix for the proof.
\end{proof}

\begin{proposition}
Let the sequence ${\bx_t}$ be defined by the update (\ref{QRDA}), $SynQ_t$ defined by Eq.~(\ref{SynQ_t}), $\bq_t$ defined by Eq.~(\ref{Qt}), $\mathbb{E}[\bq_t]=f_t(\bx_t)$. For $\forall \bg \in \mathbb{R}^d$, let $\psi_t^*(\bg)$ be the conjugate dual of $t\phi(\bx)+\frac{1}{\eta}\psi_t(\bx)$,
$\phi(\bx)=\lambda||\bx||_1$,
$||\cdot||_{\psi_t^*}=\sqrt{<\cdot,\frac{\eta}{2\bH_t}\cdot>}$. $x^*$ is the optimal solution to $f(\bx)$, for any $\bx^*\in\chi$, we have 
\begin{equation*}
    \begin{split}
        &\quad\sum_{t=1}^{T}\mathbb{E}_\bq [f_t(\bx_{t+1})+\phi(\bx_{t+1})-f_t(\bx^*)-\phi(\bx^*)]\\
        &\leq  \frac{1}{\eta}\mathbb{E}_\bq [\psi_T(\bx^*)]+\frac{\eta}{2}\sum_{t=1}^{T}\mathbb{E}_\bq||\bq_t||^2_{\psi^*_{t-1}}\\
    \end{split}
\end{equation*}
\end{proposition}
\begin{proof}
See Appendix for the proof.
\end{proof}
 \textbf{Remark:}   Since $\mathbb{E}_\bq||\bq_t-f'_t(\bx_t)||^2_{\psi^*_t}$ is the quantization variance scaled by the adaptive learning rate and 
 \begin{equation*}
    \begin{split}
\mathbb{E}_\bq||\bq_t-f'_t(\bx_t)||^2_{\psi^*_t}&\leq \mathbb{E}_\bq||\bq_t||^2_{\psi^*_t}\\
&\leq \mathbb{E}_\bq||\bq_t||^2_{\psi^*_{t-1}}
    \end{split}
\end{equation*}
We can simply regard $\mathbb{E}_\bq||\bq_t||^2_{\psi^*_t} $ as the  error introduced by quantization for QCMD adagrad and regard $\mathbb{E}_\bq||\bq_t||^2_{\psi^*_{t-1}} $  as the  error introduced by quantization for QRDA adagrad.  Therefore, Proposition 1. and Proposition 2. show that gradient quantization introduces additional noise which affects  the  model  convergence  and  sparsity. 

\subsection{Threshold Quantization}
Although gradient quantization can reduce the cost of gradient communication in distributed training, it also introduces additional errors, which affect the convergence of the model and the sparsity of parameters in QCMD adagrad and QRDA adagrad.
 As an unbiased gradient quantization method,  TernGrad \cite{wen2017terngrad} has already made a good balance between the encoding cost and accuracy of the general model. However, when it comes to the sparse model, a large quantization error still leads to slower convergence of the $l_1$ norm as a part of the objective function, which affects the sparsity of the model. In order to mitigate this problem, we apply the threshold quantization method to the QCMD adagrad and QRDA adagrad.

\textbf{Threshold quantization} is an existing quantization method used for model quantization in \cite{TWN}. We apply it to gradient quantization since it produces less error than Terngrad\cite{wen2017terngrad}.
In this section, we use $\bv^t$ to represent the (stochastic) gradient in the $t_{th}$ iteration.
Fig.~\ref{thresh_Q} gives a brief explanation of threshold quantization, and more analysis is provided below. 
\begin{figure}[h]
    \centering
    \includegraphics[width=0.7\linewidth]{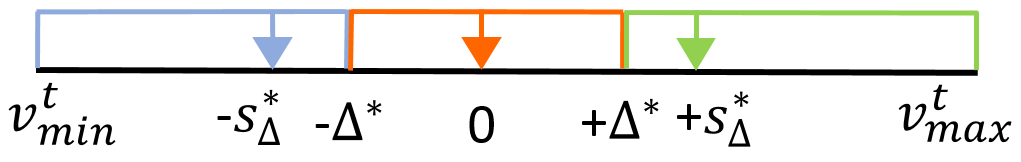}
    \caption{ An illustration of threshold quantization. Suppose $v^t_{min}\leq v^t_i\leq v^t_{max}$. $\triangle^*$ denotes the optimal threshold. For $v^t_i$ within the orange line, $v^t_i$ is quantized to 0. For $v^t_i$ within the blue line, $v^t_i$ is quantized to $-s_\triangle^*$. For $v^t_i$ within the green line, $v^t_i$ is quantized to $s_\triangle^*$.}
    \label{thresh_Q}
\end{figure}

$Q_\triangle(\cdot)$ is the threshold quantization function defined as below
\begin{equation}
Q_{\triangle}(\bv^t)=s\bv,
\end{equation}
where $\bv$ is a ternary vector, $s$ is a non-negative scaling factor and $\triangle$ denotes the threshold.
For the $i^{th}$ component of $\bv$,
\begin{equation}
\bv_i=\left\{
\begin{array}{rcl}
+1,& &{\text{if } \bv_i^t>\Delta};\\
0,& &{\text{if } |\bv_i^t|\leq\Delta};\\
-1,& &{\text{if } \bv_i^t<-\Delta}.\\
\end{array}
\right.
\end{equation}
The error 
${\epsilon}_t$ is defined as the difference between the full precision vector $\bv^t$ and $Q_\triangle(\bv^t)$:
\begin{equation}
\begin{split}
{\epsilon}_t=\bv^t-Q_\triangle(\bv^t).
\end{split}    
\end{equation}
In order to keep as more information of $\bv^t$ as possible,
the quantization method is required to minimize the 
Euclidean distance between $\bv^t$ and $Q_\triangle(\bv^t)$, i.e.,
\begin{equation}
\left\{
\begin{array}{lr}
s^*,\bv^*=\arg min_{s,\bv}||{\epsilon_t}||_2^2\\
s.t.\quad s\geq 0, \bv_i\in\{-1,0,1\}, i=1,2,...,n.
\end{array}
\right.
\label{obj_func}
\end{equation}
Obviously, the above problem can be transformed to the following formulation:
\begin{equation}
s^*,\Delta^*=\arg min_{s\geq 0,\Delta>0}|I_\Delta|s^2-2s\sum_{i\in I_\Delta}|\bv^t_i|+\sum_{i=1}^{n}(\bv^t_i)^2,
\end{equation}
where $I_\Delta=\{i||\bv^t_i|>\Delta\}$ and $|I_{\Delta}|$ denotes the number of elements in $I_\Delta$. Thus, for any given $\Delta$, the optimal $s$ can be computed as follows,
\begin{equation}
s^*_\Delta=\frac{1}{|I_\Delta|}\sum_{i\in I_\Delta}|\bv^t_i|.
\end{equation}
The optimal $\Delta$ can be computed as follows,
\begin{equation}
\Delta^*=\arg max_{\Delta>0}\frac{1}{|I_\Delta|}(\sum_{i\in I_\Delta}|\bv^t_i|)^2.
\label{objective_value}
\end{equation}

To solve the above optimal threshold, we can rank the components of the gradient and treat them as potential thresholds.
For each potential threshold, we calculate the corresponding objective value $\frac{1}{|I_\Delta|}(\sum_{i\in I_\Delta}|\bv^t_i|)^2$  
and take the potential threshold that maximizes the objective value as the optimal threshold.
The computational complexity of this process is $O(d \log d)$ for $d$-dimensional gradients, where the major cost is the sorting over all elements of gradients. 

\subsection{Threshold Approximation}\label{Threshold_approximation}
We try to improve the computational efficiency of the threshold quantization procedure without harming the optimality of the coding described in Eq.~(\ref{objective_value}). This improvement begins with the assumption that the gradient follows the Gaussian distribution.
In case of $\bv^t_i$ follows 
$N(0,\sigma^2)$, 
Li and Liu \cite{TWN} have given an approximate solution for the optimal threshold $\triangle^*$ by $0.6\sigma$ which equals $0.75\cdot\mathbb{E}(|\bv^t_i|)\approx\frac{0.75}{d}\sum_{i=1}^{d}|\bv_i^t|$. We also find in our experiment that most of the gradients satisfy the previous assumption. Fig.~\ref{histogram} shows the experimental result for training
AlexNet \cite{krizhevsky2012imagenet} on two workers.
The left column visualizes the first convolutional layer and the right one visualizes the first fully-connected layer. The distribution of the original floating gradients is close to the Gaussian distribution for both convolutional and fully-connected layers. Based on this observation, we simply use $\frac{0.75}{d}\sum_{i=1}^{d}|\bv_i^t|$ to approximate the optimal threshold to avoid the expensive cost of solving the optimal threshold $\Delta^*$ every iteration.

\textbf{Encode}. The gradient vectors need to be encoded after the threshold quantization. Specifically, we use 2 bits to encode $\bv_i$, and one floating-point number to represent the scaling factor s. For the dense model, the overall communication cost is $(32+2d)$ bits, where $d$ denotes the parameter dimension. The communication cost of threshold quantization is the same as Terngrad. But for the sparse model, assume the number of non-zero parameters is $k$ $(k\ll d)$, the communication cost of QCMD adagrad and QRDA adagrad is $(32+d+2k)$ bits. Since we need at least $d$ bits to indicate which components of the parameter are non-zero.
\begin{figure}[h] 
\centering 
\subfigure[First convolutional layer]{ \begin{minipage}[b]{0.45\linewidth} \includegraphics[width=1\linewidth]{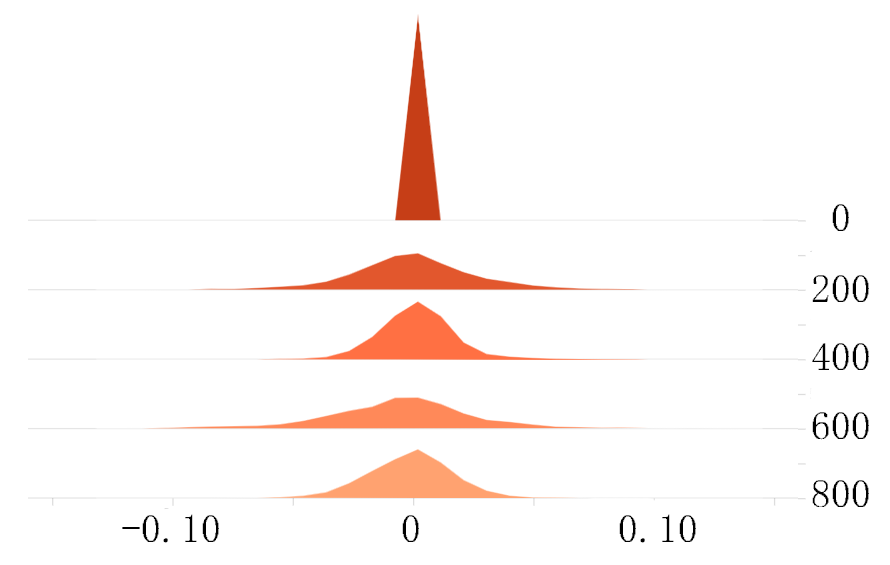}\vspace{4pt}\end{minipage}} 
\subfigure[First fully connected layer]{ 
\begin{minipage}[b]{0.45\linewidth} \includegraphics[width=1\linewidth]{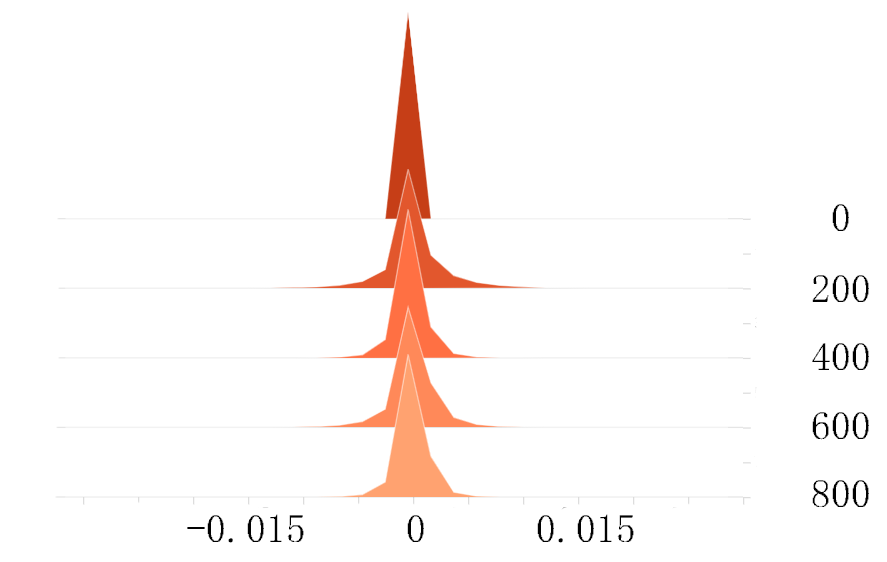}\vspace{4pt} \end{minipage}} 
\caption{Histograms of original floating gradients. 
These histograms are obtained from distributed training AlexNet on two workers, where the vertical axis is the training iteration. The left column visualizes the first convolutional layer and the right one visualizes the first fully-connected layer.} 
\label{histogram}
\end{figure}

\begin{figure*}[!htbp]
\centering 
    \includegraphics[width=\linewidth]{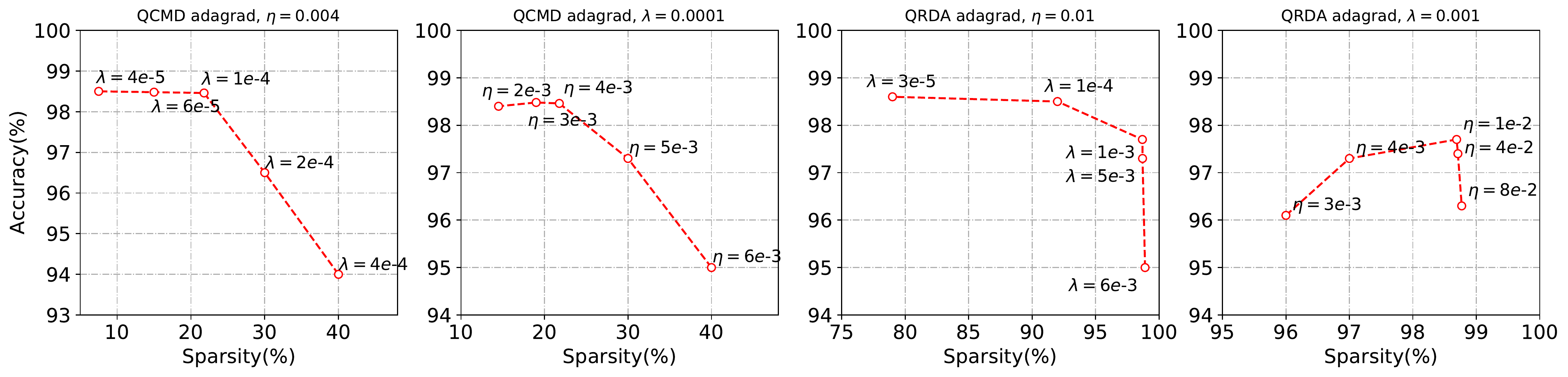}
\caption{The effect of coefficient of regularzation $\lambda$ and learning rate $\eta$ on model sparsity respectively on MNIST for QCMD adagrad and QRDA adagrad.}
\label{rw_lr_on_spr}
\end{figure*}

\section{Convergence Analysis\label{sct_con_rate}}
Two aspects are taken into account to evaluate the distributed optimization algorithm, the number of bits sent and received by the workers (communication complexity) and the number of parallel iterations required for convergence
(round complexity). In this section, we theoretically analyze the proposed QCMD adagrad and QRDA adagrad in terms of the convergence rate of regret.


To obtain the regret bound for QCMD adagrad and QRDA adagrad, we provide the follow lemma, which has been proved in Lemma 4 of \cite{duchi2011adaptive}.
\begin{lemma} For any $\delta \geq 0$, the Mahalanobis norm $||\cdot||_{\psi_t}=\sqrt{<\cdot,\bH_t\cdot>}$ and $||\cdot||_{\psi^*_t}=\sqrt{<\cdot,\frac{1}{\bH_t}\cdot>}$ be the associated dual norm, we have
 \begin{equation*}
    \begin{split}
\frac{1}{2}\sum_{t=1}^T||\bq_t||^2_{\psi^*_t}\leq \sum_{i=1}^d ||\bq_{1:T,i}||_2
    \end{split}
\end{equation*}
\end{lemma}
\begin{lemma} For any $\delta \geq \max_t||\bq_t||_{\infty}$, $||\cdot||_{\psi^*_t}=\sqrt{<\cdot,\frac{\eta}{2\bH_t}\cdot>}$, we have
 \begin{equation*}
    \begin{split}
\frac{1}{2}\sum_{t=1}^T||\bq_t||^2_{\psi^*_{t-1}}\leq \sum_{i=1}^d ||\bq_{1:T,i}||_2
    \end{split}
\end{equation*}
\end{lemma}

Combining the above arguments with Proposition.1 and  Proposition.2, we have the following theorem.

\begin{theorem}
For QCMD adagrad, let $D_\infty=\max_{t\leq T}||\bx^*-\bx_t||_\infty$, $G_\infty=\max_{t\leq T,i\leq d}||\bq_{1:T,i}||_2$, 
the regularizer $\phi(\bx)=\lambda||\bx||_1$, where $\lambda\geq 0$. Assume $Q(\cdot)$ is an unbiased quantization function,  $f_t(\bx)$ is L-smooth function, learning rate $\eta=\frac{2}{1+2L}$,  $\psi_t(x)$ is 1-strongly convex function, the regret is as below
\begin{equation*}
\begin{split}
&\quad\frac{1}{T}
\sum_{t=1}^{T}\mathbb{E}_\bq [f_t(\bx_{t+1})+\phi(\bx_{t+1})-f_t(\bx^*)-\phi(\bx^*)]\\
&\leq \frac{dG_\infty}{\sqrt{T}}+\frac{dG_\infty D_\infty}{2\eta\sqrt{T}}
\end{split}
\end{equation*}
\end{theorem}
\begin{proof}
See Appendix for the proof.
\end{proof}
\begin{theorem}
For QRDA adagrad, let $D_\infty=\max_{t\leq T}||\bx^*-\bx_t||_\infty$, $G_\infty=\max_{t\leq T,i\leq d}||\bq_{1:T,i}||_2$, $Q(\cdot)$ be an unbiased quantization function, the regularizer $\phi(\bx)=\lambda||\bx||_1$, where $\lambda\geq 0$. Then the regret is as below
\begin{equation*}
    \begin{split}
        &\quad\frac{1}{T}\sum_{t=1}^T\mathbb{E}_\bq[f_t(\bx_t)+\phi(\bx_t)-f_t(\bx^*)-\phi(\bx^*)]\\
        &\leq \frac{\delta||\bx^*||_2^2}{\eta T}+\frac{(\frac{1}{\eta}||\bx^*||^2_\infty+\frac{\eta^2}{2}) dG_\infty }{\sqrt{T}}
    \end{split}
\end{equation*}
\end{theorem}

\begin{proof}
See Appendix for the proof.
\end{proof}

{\bf Remark:}
The convergence rate of QCMD adagrad and QRDA adagrad is comparable with CMD adagrad and RDA adagrad respectively, in terms of $O(\frac{1}{\sqrt{T}})$. The quantization error affects the convergence. Only the quantization error is sufficiently small, the proposed methods behave similarly with CMD adagrad and RDA adagrad. The convergence of $f(\bx)$ determines the accuracy of models and  $\phi(\bx)$ determines the sparsity of models. In QCMD adagrad, the learning rate $\eta$ and the coefficient of regularization $\lambda$ control the sparsity of models, while QRDA adagrad controls the sparsity of models only through $\lambda$.

\section{Experiments\label{sct_exp}}
\subsection{Experimental Settings}



\begin{table}[ht]
    \caption{Data statistics}  
    \centering
    \begin{tabular}{cccc}
  \hline
        \multirow{2}*{Data Set}& Number of &Number of &Dimension of\\ 
    ~ & Training Data & Test Data & Input Feature \\ 
  \hline
  news20&15,994&4,000&1,355,191\\
        rcv1&20,242&677,399&47,236\\ 
        MNIST&60,000&10,000&1*28*28\\
        CIFAR-10&50,000&10,000&3*32*32\\ 
        \hline 
        
    \end{tabular}\vspace{-0.1in}
    \label{dataset}
\end{table}

In this section, we first conduct experiments on linear models to validate the effectiveness and efficiency of the proposed QCMD adagrad and QRDA adagrad for binary classification problems. After that, we use our proposed method to train convolutional neural networks, in order to validate the performance of them on non-convex problems. 

\textbf{Baselines}. We compare QCMD adagrad and QRDA adagrad with several sparse model distributed optimization methods, including 32 bits Prox-gd\cite{duchi2010composite}, 32 bits CMD adagrad\cite{duchi2011adaptive}, 32 bits RDA adagrad\cite{duchi2011adaptive} and their corresponding ternary variants\cite{wen2017terngrad}. ${\dagger}$ is to mark the methods that only  the local gradients are quantized but the synchronous gradient.

\textbf{Implementation Details}. 
All the experiments are carried out in a distributed framework that the network bandwidth is 100MBps. For the linear models training, each worker only utilizes the CPUs. For the training of convolutional neural networks, each worker is allocated with 1 NVIDIA Tesla P40 GPU. The methods are evaluated on four publicly available datasets. \textbf{news20} and \textbf{rcv1} are text datasets with a high dimension of input features from LIBSVM \cite{chang2001libsvm}. \textbf{MNIST} \cite{lecun1998gradient} is for handwritten digits classification problem 
and \textbf{CIFAR10} \cite{krizhevsky2009learning} is for image classification problem.
Table \ref{dataset} shows the details of datasets. For \textbf{news20} and \textbf{rcv1}, the $\ell_1$ norm regularized logistic regression model are trained. As for multi-classification problems, we train LeNet for \textbf{MNIST}\cite{lecun1998gradient} and AlexNet \cite{krizhevsky2012imagenet} for \textbf{CIFAR10}. To generate a network with large sparsity, the batch normalization layer is added before each convolution layer and fully connected layer.  The code is implemented via tensorflow. Experimental results are averaged over 5 runs with a random initialization seed.

\begin{table}[!htbp]
    \caption{Settings of hyperparameters}
    \centering
     \begin{small}
    \begin{tabular}{ccccc}
        \hline 
        \multicolumn{5}{c}{\textbf{news20}}\\
        \hline
        \multirow{2}*{Method}&Learning&Coefficient of &Batch size &Number of\\
        ~&rate&regularization&per worker&worker\\
        \hline 
        Prox-gd&0.1&0.001&20&2 \\
        QCMD adagrad&0.02&0.00001&20&2\\
        QRDA adagrad&0.02&0.1&20&2\\
        \hline
        \multicolumn{5}{c}{\textbf{rcv1}}\\
        \hline        
        \multirow{2}*{Method}&Learning&Coefficient of &Batch size &Number of\\
        ~&rate&regularization&per worker&worker\\
        \hline 
        Prox-gd&1.0&0.0001&20&2 \\
        QCMD adagrad&1.0&0.000005&20&2 \\
        QRDA adagrad&0.1&0.5&20&2\\
        \hline 
        \multicolumn{5}{c}{\textbf{MNIST}}\\
        \hline        
        \multirow{2}*{Method}&Learning&Coefficient of &Batch size &Number of\\
        ~&rate&regularization&per worker&worker\\
        \hline 
        Prox-gd&0.001&0.005&16&4 \\
        QCMD adagrad&0.004&0.0001&16&4 \\
        QRDA adagrad&0.01&0.001&16&4\\
        \hline 
        \multicolumn{5}{c}{\textbf{CIFAR-10}}\\
        \hline        
        \multirow{2}*{Method}&Learning&Coefficient of &Batch size &Number of\\
        ~&rate&regularization&per worker&worker\\
        \hline 
        Prox-gd&0.01&0.001&64&2 \\
        QCMD adagrad&0.01&0.0002&64&2\\
        QRDA adagrad&0.01&0.0004&64&2\\
        \hline 
    \end{tabular}
     \end{small}
        \label{setting}
\end{table}

\begin{table}[!htbp] 
    \caption{Comparisons on four datasets in terms of three metrics. ${\dagger}$ is to mark the methods that only quantize the local gradient but the synchronous gradient. }
    \centering
    \begin{small}
    \begin{tabular}{cccc}
        \hline 
        \multicolumn{4}{c}{\textbf{news20}}\\
        \hline

        \textbf{Method}&\textbf{Accuracy(\%)}&\textbf{Sparsity(\%)}&\textbf{Error}\\
        \hline 
        32-bits Prox-gd&    73.92&    56.63&-\\
        32-bits CMD adagrad&    \textbf{97.50}&    85.76&-\\
        Ternary CMD adagrad$^{\dagger}$&    96.50&    82.64&3.13e-09\\
        Threshold CMD adagrad$^{\dagger}$&    97.46&    84.59&\textbf{1.49e-09}\\
        Ternary CMD adagrad&    96.00&    80.30&5.24e-09\\
        Threshold CMD adagrad&    97.19&83.83&2.39e-09\\
        32-bits RDA adagrad&    97.21&    98.21&-\\
        Ternary RDA adagrad$^{\dagger}$&    95.88&    97.52&1.63e-08\\
        Threshold RDA adagrad$^{\dagger}$&    97.16&    \textbf{98.26}&2.92e-09\\
        Ternary RDA adagrad&    95.50&    97.05& 2.65e-08\\
        Threshold RDA adagrad&    97.09&    98.18&2.93e-09\\
        \hline 
        \multicolumn{4}{c}{\textbf{rcv1}}\\
        \hline        
        \textbf{Method}&\textbf{Accuracy(\%)}&\textbf{Sparsity(\%)}&\textbf{Error}\\
        \hline 
        32-bits Prox-gd&    91.96&    33.87&-\\
        32-bits CMD adagrad&    \textbf{95.15}&    72.66&-\\
        Ternary CMD adagrad$^{\dagger}$&    94.97&    71.70&5.07e-08\\
        Threshold CMD adagrad$^{\dagger}$&    95.16&    73.76&\textbf{1.40e-08}\\
        Ternary CMD adagrad&    94.87&    70.12&7.77e-08\\
         Threshold CMD adagrad&    94.99&    73.18&1.45e-08\\
        32-bits RDA adagrad&    93.41&    \textbf{97.56}&-\\
        Ternary RDA adagrad$^{\dagger}$&    91.14&    96.86&1.73e-07\\
        Threshold RDA adagrad$^{\dagger}$&    93.21&    97.54
&2.32e-08\\
        Ternary RDA adagrad&    90.90&    96.43 &3.03e-07\\
        Threshold RDA adagrad&    93.01&    97.39&2.62e-08\\
        \hline 
        \multicolumn{4}{c}{\textbf{MNIST}}\\
        \hline        
        \textbf{Method}&\textbf{Accuracy(\%)}&\textbf{Sparsity(\%)}&\textbf{Error}\\
        \hline 
        32-bits Prox-gd&    94.17&    7.86&-\\
        32-bits CMD adagrad&    98.28&    17.19&-\\
        Ternary CMD adagrad$^{\dagger}$&    98.22&    9.61&3.52e-3\\
        Threshold CMD adagrad$^{\dagger}$&    98.34&    22.08&\textbf{2.77e-3}\\
        Ternary CMD adagrad&    97.45&    10.19&4.82e-3\\
        Threshold CMD adagrad&    \textbf{98.46}&    21.78&2.95e-3\\
        32-bits RDA adagrad&    97.85&    98.27&-\\
        Ternary RDA adagrad$^{\dagger}$&    97.57&    95.27&1.30e-1\\
        Threshold RDA adagrad$^{\dagger}$&    97.84&    98.59&1.98e-2\\
        Ternary RDA adagrad&    97.32&    90.39&1.51e-1\\
        Threshold RDA adagrad&    97.70&    \textbf{98.69}&2.32e-2\\
        \hline 
        \multicolumn{4}{c}{\textbf{CIFAR-10}}\\
        \hline        
        \textbf{Method}&\textbf{Accuracy(\%)}&\textbf{Sparsity(\%)}&\textbf{Error}\\
        \hline 
        32-bits Prox-gd&    72.73&    76.51&-\\
        32-bits CMD adagrad&    \textbf{86.63}&    80.50&-\\
        Ternary CMD adagrad$^{\dagger}$&    85.27&    0.42& 1.18e-1\\
        Threshold CMD adagrad$^{\dagger}$&    86.23&    62.51&2.45e-2\\
        Ternary CMD adagrad&    84.37&    0.20 & 2.58e-1\\
        Threshold CMD adagrad&    84.74&    39.95&2.60e-2\\
        32-bits RDA adagrad&    86.49&    87.11&-\\
        Ternary RDA adagrad$^{\dagger}$&    84.09&    80.05&8.70e-1\\
        Threshold RDA adagrad$^{\dagger}$&    85.57&    \textbf{87.65}&\textbf{2.17e-2}\\
        Ternary RDA adagrad&    82.84&    75.46&9.90e-1\\
        Threshold RDA adagrad&    84.22&    81.86&2.77e-2\\
        \hline 
    \end{tabular}
     \end{small}
        \label{result}
       
\end{table}

\begin{figure*}[th]  
    \centering
     \includegraphics[width=\linewidth]{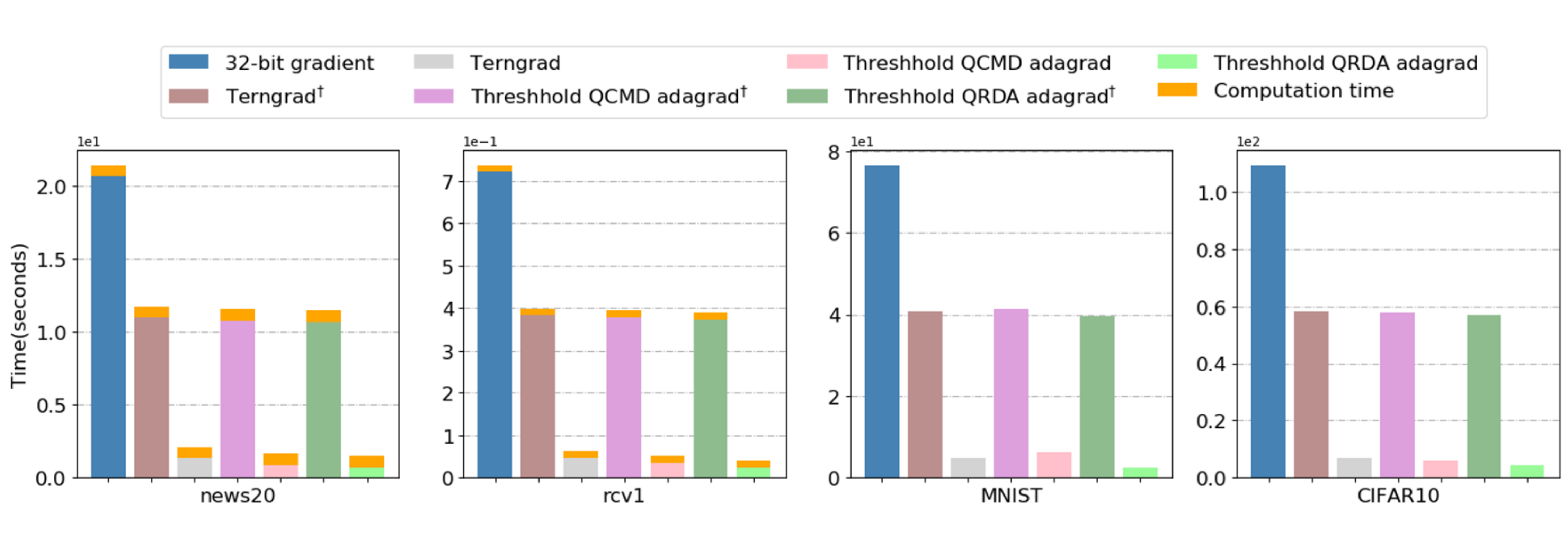}
    \caption{Comparison on the decomposed time consumption for training logistic regression on news20, rcv1 and training Lenet on MNIST, AlexNet on CIFAR10. Each histogram shows one-step decomposed time consumption averaged over the entire training procedure. ${\dagger}$ is to mark the methods that only quantize the local gradient but the synchronous gradient.
    The lower rectangles are the transmission time consumption and the upper rectangles are the computation time consumption.
    The network bandwidth between the parameter server and workers is 100MBps. Since LeNet and AlexNet are trained on GPUs, the computation time consumption is insignificant for MNIST and CIFAR10.
    }    
    \label{time}
    
\end{figure*}

\textbf{Hyperparameter Discussion}. Table \ref{setting} lists the related hyperparameters of
the sparse model training in distributed settings, including the learning rate $\eta$, coefficient of regularization $\lambda$, batch size per worker and the number of the worker.
The optimal choice of  $\eta$ and  $\lambda$  can vary somewhat from dataset to dataset. The structure and random initial of the network also affect them.  In order to find the appropriate hyperparameters, we first select the best learning rate based on cross-validations without the regularization and then find the coefficient of regularization that maximizes the sparsity of the model without reducing the accuracy as much as possible. For ease of comparison, we set the same hyper parameters for corresponding ternary quantization methods. Fig. \ref{rw_lr_on_spr} shows the accuracy and model sparsity as the $\eta$ and $\lambda$ change on MINIST.  The sparsity increase with the $\lambda$.
A strong coefficient of regularization $\lambda$ enforces a highly sparse model which may deteriorate the accuracy. When the learning rate is high or $\lambda$ is large, it is easy to cause oscillation for QCMD adagrad.
For QCMD adagrad, the model sparsity increases not only with the coefficient of regularization $\lambda$ but also the learning rate $\eta$. For QRDA adagrad, model sparsity is mainly affected by $\lambda$, while the learning rate has less effect.

\subsection{Results and Discussions}

We summarize three metrics including  the accuracy, model sparsity and error of the quantized gradient in Table \ref{result}, and report the time consumption in Fig.~\ref{time}. From these results, we draw several conclusions. 

\textbf{Firstly}, the sparsity of model on the four datasets is 98.18\%, 97.39\%, 98.69\% and 81.86\% for QRDA adagrad, respectively. QRDA adagrad is easier to generate extremely sparse models than QCMD adagrad, since 
it uses the same evaluation criteria $\lambda$ for each dimension of the model parameters to determine whether to retain the parameter values. In the case of highly sparse model for QRDA adagrad, the accuracy decreases slightly compared with QCMD adagrad whose sparseness of model is relatively low. For example, on MNIST, the accuracy for threshold QRDA adagrad achieves 97.70\%, which is $0.76\%$ lower than that of threshold QCMD adagrad. 

\textbf{Secondly}, although the sparsity of the model generated by QCMD adagrad is limited. However, it  still performs better than 32 bits Prox-gd in accuracy and sparsity,  since the proximal function $\psi$ use the Bregman divergence to keep the parameter $\bx_{t+1}$ close to $\bx_t$ and it adaptively retains the parameter values for each dimension.  

\textbf{Thirdly}, the noise introduced by the quantization method affects the convergence of accuracy and sparsity of the model. The one-way quantization (only the local gradients are quantized instead of both local gradients and the synchronous gradient) introduces less noise than the double quantization.
The error of threshold quantization is smaller than Terngrad's and it preserves sparsity and accuracy of the model both for QCMD adagrad and QRDA adagrad. Specifically,  for threshold quantization, the accuracy and sparsity of the model for double quantized scheme are relatively close to 32 bits gradient scheme's, while ternary quantization  reduces the sparsity and accuracy of the model on these 4 datasets. 



\textbf{Lastly}, double quantization reduces most of the communication cost.  Fig.~\ref{time} decomposes the time consumption. For example, double quantized threshold QRDA can save $96.06\%$  of the total training time on CIFAR10. Threshold QRDA (only the local gradients are quantized) saves $48.03\%$ of the total training time on CIFAR10. The QRDA adagrad consumes the least communication cost since it generates a highly sparse model.


\section{Conclusion \label{sct_conclusion}}
In this paper, we present distributed QCMD adagrad and QRDA adagrad to accelerate large-scale machine learning. QCMD adagrad and QRDA adagrad combine the gradient quantization and sparse model to reduce the communication cost per iteration in distributed training. We Theoretical analysis  that the convergence rate of QCMD adagrad and QRDA adagrad is comparable with CMD adagrad and RDA adagrad respectively. Considering that a large error of quantization gradients affects the convergence of QCMD adagrad and QRDA adagrad, we adopt the threshold quantization with less quantized error to preserve the model performance and sparsity.
The encouraging empirical results   on  linear models and convolutional neural networks demonstrate the efficacy
and efficiency of the proposed algorithms in balancing communication cost, accuracy, and model sparsity.

{\clearpage
\balance


}

{\clearpage
\appendix









\section{Appendeix} 

\subsection{Proof of Proposition 1}
\begin{proof}
Assume $f_t(\bx)$ is L-smoothness, one obtains 
\begin{equation*}
\begin{split}
 f_t(\bx_{t+1})-f_t(\bx_t) \leq f'_t(\bx_t)^\top(\bx_{t+1}-\bx_t)+\frac{L}{2}||\bx_{t+1}-\bx_t||^2\\
\end{split}    
\end{equation*}
Let $\psi_t(\bx)=\frac{1}{2}\bx^\top \bH_t \bx$, the Mahalanobis norm $||\cdot||_{\psi_t}=\sqrt{<\cdot, \bH_t \cdot>}$ and $||\cdot||_{\psi_t^*}$ be the associated dual norm. Introduce the quantized gradient, we have
\begin{equation*}
\begin{split}
&\quad f_t(\bx_{t+1})-f_t(\bx_t)\\
&\leq \bq_t^\top(\bx_{t+1}-\bx_t)+\frac{L}{2}||\bx_{t+1}-\bx_t||^2+(f'_t(\bx_t)- \bq_t)^\top(\bx_{t+1}-\bx_t)\\
&\leq \bq_t^\top(\bx_{t+1}-\bx_t)+\frac{L}{2}||\bx_{t+1}-\bx_t||^2+\frac{1}{2}||\bx_{t+1}-\bx_t||^2_{\psi_t}\\
&\quad+\frac{1}{2}||\bq_t-f'_t(\bx_t)||^2_{\psi^*_t}\\
\end{split}    
\end{equation*}
 The second inequality follows from Cauchy-Schwarz (with $2ab\leq xa^2+\frac{b^2}{x}$ for any $x>0$). $B_{\psi_t}(\bx_{t+1},\bx_t)=\frac{1}{2}(\bx_{t+1}-\bx_t)^\top \bH_t(\bx_{t+1}-\bx_t)=\frac{1}{2}||\bx_{t+1}-\bx_t||^2_{\psi_t}$ is Bregman divergence.
Let $\psi_t$ be a mirror map 1-strongly convex on $\chi$, then
\begin{equation*}
0\leq\frac{1}{2}||\bx_{t+1}-\bx_t||^2\leq B_{\psi_t}(\bx_{t+1},\bx_t)
\end{equation*}
So,
\begin{equation*}
\begin{split}
\quad f_t(\bx_{t+1})&-f_t(\bx_t)\leq \bq_t^\top(\bx_{t+1}-\bx_t)\\
&+\frac{1}{2}||\bq_t-f'_t(\bx_t)||^2_{\psi^*_t}+(\frac{1}{2}+L)B_{\psi_t}(\bx_{t+1},\bx_t)
\end{split}    
\end{equation*}
Since $\psi_t'(\bx_t)-\psi_t'(\bx_{t+1})=\bH_t(\bx_t-\bx_{t+1})=\eta \bq_t+\eta\phi'(\bx_{t+1})$,
\begin{equation*}
\begin{split}
    &\quad\eta[\phi'(\bx_{t+1})+ \bq_t]^\top(\bx_{t+1}-\bx^*)\\
    &=\left(\psi_t'(\bx_t)-\psi_t'(\bx_{t+1})\right)^\top(\bx_{t+1}-\bx^*)\\
    &=B_{\psi_t}(\bx^*,\bx_t)-B_{\psi_t}(\bx^*,\bx_{t+1})-B_{\psi_t}(\bx_{t+1},\bx_t)
\end{split}
\end{equation*}
If we set the learning rate $\eta=\frac{2}{1+2L}$, we have 
\begin{equation*}
\begin{split}
&\quad f_t(\bx_{t+1})-f_t(\bx_t)\\
&\leq \bq_t^\top(\bx_{t+1}-\bx_t)+\frac{1}{2}||\bq_t-f'_t(\bx_t)||^2_{\psi^*_t}\\
&\quad+\frac{1}{\eta}(B_{\psi_t}(\bx^*,\bx_t)-B_{\psi_t}(\bx^*,\bx_{t+1}))\\
&\quad-\bq_t^\top(\bx_{t+1}-\bx^*)-\phi'(\bx_{t+1})^\top(\bx_{t+1}-\bx^*)\\
&=\bq_t^\top(\bx^*-\bx_t)+\frac{1}{2}||\bq_t-f'_t(\bx_t)||^2_{\psi^*_t}\\
&\quad+\frac{1}{\eta}(B_{\psi_t}(\bx^*,\bx_t)-B_{\psi_t}(\bx^*,\bx_{t+1}))\\
&\quad-\phi'(\bx_{t+1})^\top(\bx_{t+1}-\bx^*)
\end{split}    
\end{equation*}
From $f_t(\bx^*)+\phi(\bx^*)-[f_t(\bx_t)+\phi(\bx_t)+(f'_t(\bx_t)+\phi'(\bx_t))^\top(\bx^*-\bx_t)]\geq 0$,
\begin{equation*}
\begin{split}
&\quad f_t(\bx_{t+1})+\phi(\bx_{t+1})-f_t(\bx^*)-\phi(\bx^*)\\
&\leq \bq_t^\top(\bx^*-\bx_t)-f'_t(\bx_t)^\top(\bx^*-\bx_t)+\phi'(\bx_{t+1})^\top(\bx_{t+1}-\bx^*)\\
&\quad+\frac{1}{2}||\bq_t-f'_t(\bx_t)||^2_{\psi^*_t}\\
&\quad+\frac{1}{\eta}(B_{\psi_t}(\bx^*,\bx_t)-B_{\psi_t}(\bx^*,\bx_{t+1}))-\phi'(\bx_{t+1})^\top(\bx_{t+1}-\bx^*)\\
\end{split}    
\end{equation*}
If $\mathbb{E}[\bq_t]=f'_t(\bx_t)$, we take expectation on both sides of the inequality, we obtains
\begin{equation*}
\begin{split}
&\quad\mathbb{E}_\bq [f_t(\bx_{t+1})+\phi(\bx_{t+1})-f_t(\bx^*)-\phi(\bx^*)]\\
&\leq 0 +\frac{1}{2}\mathbb{E}_\bq||\bq_t-f'_t(\bx_t)||^2_{\psi^*_t}+\frac{1}{\eta}\mathbb{E}_\bq(B_{\psi_t}(\bx^*,\bx_t)-B_{\psi_t}(\bx^*,\bx_{t+1}))\\
&\leq \frac{1}{2}\mathbb{E}_\bq||\bq_t||^2_{\psi^*_t}+\frac{1}{\eta}\mathbb{E}_\bq(B_{\psi_t}(\bx^*,\bx_t)-B_{\psi_t}(\bx^*,\bx_{t+1}))\\
\end{split}    
\end{equation*}
Sum the inequality,
\begin{equation*}
\begin{split}
 &\quad\sum_{t=1}^{T}\mathbb{E}_\bq [f_t(\bx_{t+1})+\phi(\bx_{t+1})-f_t(\bx^*)-\phi(\bx^*)]\\
&\leq \frac{1}{2}\sum_{t=1}^{T}\mathbb{E}_\bq||\bq_t||^2_{\psi^*_t}+\frac{1}{\eta}\sum_{t=1}^{T}\mathbb{E}_\bq(B_{\psi_t}(\bx^*,\bx_t)-B_{\psi_t}(\bx^*,\bx_{t+1}))\\
&= \frac{1}{2}\sum_{t=1}^{T}\mathbb{E}_\bq||\bq_t||^2_{\psi^*_t}+\frac{1}{\eta}\mathbb{E}_\bq B_{\psi_1}(\bx^*,\bx_1)\\
&\quad+\frac{1}{\eta}\sum_{t=1}^{T-1}\mathbb{E}_\bq[B_{\psi_{t+1}}(\bx^*,\bx_{t+1})-B_{\psi_t}(\bx^*,\bx_{t+1})]\\
\end{split}    
\end{equation*}
\end{proof}

\subsection{Proof of Theorem 1}
\begin{proof}
The third item in the conclusion of Proposition 1 can be transformed as below:
\begin{equation*}
\begin{split}
&\frac{1}{\eta}\sum_{t=1}^{T-1}\mathbb{E}_\bq[B_{\psi_{t+1}}(\bx^*,\bx_{t+1})-B_{\psi_{t}}(\bx^*,\bx_{t+1})]\\
=&\frac{1}{\eta}\sum_{t=1}^{T-1}\mathbb{E}_\bq[\frac{1}{2}(\bx^*-\bx_{t+1})^\top(\bH_{t+1}-\bH_t)(\bx^*-\bx_{t+1})]\\
\leq&\frac{1}{2\eta}\max_{t\leq T}||\bx^*-\bx_t||^2_\infty\sum_{i=1}^{d}\mathbb{E}_\bq||f'(\bx)_{1:T,i}||_2\\
&\quad-\frac{1}{2\eta}||\bx^*-\bx_1||_\infty^2\mathbb{E}_\bq<\bh_1,\bf{1}>\\
\end{split}
\end{equation*}
$\bf{1}$ is a vector whose elements are all 1. Since $\mathbb{E}_\bq B_{\psi_{1}}(\bx^*,\bx_1)\leq \frac{1}{2}||\bx^*-\bx_1||^2_\infty\mathbb{E}_\bq<\bh_1,\bf{1}>$,
\begin{equation*}
\begin{split}
&\quad\sum_{t=1}^{T}\mathbb{E}_\bq [f_t(\bx_{t+1})+\phi(\bx_{t+1})-f_t(\bx^*)-\phi(\bx^*)]\\
&\leq  \frac{1}{2}\sum_{t=1}^{T}\mathbb{E}_\bq||\bq_t||^2_{\psi^*_t}+\frac{1}{2\eta}\max_{t\leq T}||\bx^*-\bx_t||^2_\infty\sum_{i=1}^{d}\mathbb{E}_\bq||\bq_{1:T,i}||_2\\
\end{split}    
\end{equation*}
$\frac{1}{2}\sum_{t=1}^{T}\mathbb{E}_\bq||\bq_t||^2_{\psi^*_t}\leq \sum_{i=1}^{d}\mathbb{E}_\bq||\bq_{1:T,i}||_2$ has been proved in Lemma 4 of \cite{duchi2011adaptive}, then
\begin{equation*}
\begin{split}
&\quad\sum_{t=1}^{T}\mathbb{E}_\bq [f_t(\bx_{t+1})+\phi(\bx_{t+1})-f_t(\bx^*)-\phi(\bx^*)]\\
&\leq \sum_{i=1}^{d}\mathbb{E}_\bq||\bq_{1:T,i}||_2+\frac{1}{2\eta}\max_{t\leq T}||\bx^*-\bx_t||^2_\infty\sum_{i=1}^{d}\mathbb{E}_\bq||\bq_{1:T,i}||_2\\
\end{split}    
\end{equation*}
We define that $D_\infty=\max_{t\leq T}||\bx^*-\bx_t||_\infty$, $G_\infty=\max_{t\leq T,i\leq d}||\bq_{1:T,i}||_2$.
\begin{equation*}
\begin{split}
&\quad\sum_{t=1}^{T}\mathbb{E}_\bq [ f_t(\bx_{t+1})+\phi(\bx_{t+1})-f_t(\bx^*)-\phi(\bx^*)]\\
&\leq d\sqrt{T}G_\infty+\frac{1}{2\eta}D_\infty d\sqrt{T}G_\infty\\
\end{split}    
\end{equation*}
Jensen's inequality:
\begin{equation*}
\begin{split}
&\quad\mathbb{E}_\bq[f(\bar{\bx})+\phi(\bar{\bx})]-f(\bx^*)-\phi(\bx^*)]\\
&\leq\frac{1}{T}
\sum_{t=1}^{T}\mathbb{E}_\bq [ f_t(\bx_{t+1})+\phi(\bx_{t+1})-f_t(\bx^*)-\phi(\bx^*)]\\
&\leq \frac{dG_\infty}{\sqrt{T}}+\frac{dG_\infty D_\infty}{2\eta\sqrt{T}}
\end{split}
\end{equation*}
\end{proof}

\subsection{Proof of Proposition 2}
\begin{proof}
Define $\psi_t^*(\bg)$ to be the conjugate dual of $t\phi(\bx)+\frac{1}{\eta}\psi_t(\bx)$:
\begin{equation}\label{psi}
    \begin{split}
       \psi^*_t(\bg)=\sup_{\bx\in\chi}\{<\bg,\bx>-t\phi(\bx)-\frac{1}{\eta}\psi_t(\bx)\}
    \end{split}
\end{equation}
If set $\phi(\bx)=\lambda||\bx||_1$, then
\begin{equation*}
    \begin{split}
    \psi^*_{ti}(\bg)=\left\{
\begin{array}{rcl}
(\bg_i-t\lambda)^2\frac{\eta}{4\bH_{t,ii}},      &      & {g_i-t\lambda     >0     }\\
(\bg_i+t\lambda)^2\frac{\eta}{4\bH_{t,ii}},    &      & {g_i+t\lambda < 0}\\
0,    &      & {otherwise}\\
\end{array} \right. 
    \end{split}
\end{equation*}
So, $||\cdot||_{\psi_t^*}=\sqrt{<\cdot,\frac{\eta}{2\bH_t}\cdot>}$.

Since $\frac{\psi_t}{\eta}$ is $\frac{1}{\eta}$-strongly convex with respect to the norm $||\cdot||_{\psi_t}$, the function $\psi_t^*$ is $\eta$-smooth with respect to $||\cdot||_{\psi^*_t}$.
\begin{equation*}
    \psi_t^*(\by)\leq\psi_t^*(\bx)+\triangledown\psi^*_t(\bx)(\by-\bx)+\frac{\eta}{2}||\by-\bx||^2_{\psi^*_t}
\end{equation*}
Further, a simple argument with the fundamental theorem of the conjugate dual is
\begin{equation}
    \triangledown\psi^*_t(\bg)=\arg\min_{\bx\in\chi}\{-<\bg,\bx>+t\phi(\bx)+\frac{1}{\eta}\psi_t(\bx)\}
\end{equation}
If $\bz_t=\sum_{\tau=1}^t\bq_\tau$, we have
\begin{equation}\label{fundamental_t_4_dual}
    \begin{split}
       \psi'^*_t(-\bz_t)&=\arg\min_{\bx\in\chi}\{<\sum_{\tau=1}^t\bq_\tau,\bx>+t\phi(\bx)+\frac{1}{\eta}\psi_t(\bx)\}\\
       &=\bx_{t+1}
    \end{split}
\end{equation}
To obtain the regret for quantized rda adagrad, we first calculate 
{\small\begin{equation*}
    \begin{split}
&\quad\sum_{t=1}^T[\bq_t^\top(\bx_t-\bx^*)+\phi(\bx_{t+1})-\phi(\bx^*)]\\
&\leq\sum_{t=1}^T[\bq_t^\top \bx_t+\phi(\bx_{t+1})]+ \frac{1}{\eta}\psi_T(\bx^*)+\sup_{\bx\in\chi}\left\{-\sum_{t=1}^T[\bq_t^\top \bx]-T\phi(\bx)-\frac{1}{\eta}\psi_T(\bx)\right\}\\
&=\sum_{t=1}^T[\bq_t^\top \bx_t+\phi(\bx_{t+1})]+ \frac{1}{\eta}\psi_T(\bx^*)+\psi^*_T(-\bz_T)\\
    \end{split}
\end{equation*}}
Since Equation (\ref{psi}) and Equation (\ref{fundamental_t_4_dual}), it is clear that
\begin{equation*}
    \begin{split}
      \psi_T^*(-\bz_T)&= -\sum_{t=1}^T[\bq_t^\top \bx_{T+1}]-T\phi(\bx_{T+1})-\frac{1}{\eta}\psi_T(\bx_{T+1})\\
      &\leq  -\sum_{t=1}^T[\bq_t^\top \bx_{T+1}]-(T-1)\phi(\bx_{T+1})-\phi(\bx_{T+1})-\frac{1}{\eta}\psi_{T-1}(\bx_{T+1})\\
      &\leq\sup_{\bx\in\chi}\left\{-\bz_T^\top \bx-(T-1)\phi(\bx)-\frac{1}{\eta}\psi_{T-1}(\bx)\right\}-\phi(\bx_{T+1})\\
      &=\psi^*_{T-1}(-\bz_T)-\phi(\bx_{T+1})
    \end{split}
\end{equation*}
The first inequality above follows from $\psi_{t+1}\geq\psi_t$. Further, the identity (\ref{fundamental_t_4_dual}) and the fact that $\bz_T-\bz_{T-1}=-\bq_T$ give
\begin{equation*}
    \begin{split}
        &\quad\sum_{t=1}^T[\bq_t^\top \bx_t+\phi(\bx_{t+1})]+ \frac{1}{\eta}\psi_T(\bx^*)+\psi^*_T(-\bz_T)\\
        &\leq \sum_{t=1}^T[\bq_t^\top \bx_t+\phi(\bx_{t+1})]+ \frac{1}{\eta}\psi_T(\bx^*)+\psi^*_{T-1}(-\bz_T)-\phi(\bx_{T+1})\\
        &\leq \sum_{t=1}^T[\bq_t^\top \bx_t+\phi(\bx_{t+1})]+ \frac{1}{\eta}\psi_T(\bx^*)-\phi(\bx_{T+1}) \\
        &\quad+{\psi^{*}}_{T-1}(-\bz_{T-1})-{\psi'_{T-1}}^{*}{(-{\bz_{T-1})}^\top} \bq_T +\frac{\eta}{2}||\bq_T||^2_{\psi^*_{T-1}}\\
        &=\sum_{t=1}^{T-1}[\bq_t^\top \bx_t+\phi(\bx_{t+1})]+ \frac{1}{\eta}\psi_T(\bx^*)+\psi^*_{T-1}(-\bz_{T-1})+\frac{\eta}{2}||\bq_T||^2_{\psi^*_{T-1}}\\
    \end{split}
\end{equation*}
We can repeat the same sequence of steps to see that
\begin{equation*}
    \begin{split}
        &\quad\sum_{t=1}^T[\bq_t^\top \bx_t+\phi(\bx_{t+1})]+ \frac{1}{\eta}\psi_T(\bx^*)+\psi^*_T(-\bz_T)\\
        &\leq \frac{1}{\eta}\psi_T(\bx^*)+\psi_0^*(-\bz_0)+\frac{\eta}{2}\sum_{t=1}^{T}||\bq_t||^2_{\psi^*_{t-1}}\\
        &= \frac{1}{\eta}\psi_T(\bx^*)+\frac{\eta}{2}\sum_{t=1}^{T}||\bq_t||^2_{\psi^*_{t-1}}\\
    \end{split}
\end{equation*}
Take expectation on both sides of the inequality, we get the results in proposition 2. 
\end{proof}

\subsection{Proof of Theorem 2}
\begin{proof}
We set $\delta\geq\max_t||\bq_t||^2_\infty$, $\bH_t=\delta \bI+diag(\bc_t)$, in which case 
\begin{equation*}
    \begin{split}
     ||\bq_t||^2_{\psi_{t-1}^*}
     \leq \frac{\eta}{2}\bq_t^\top diag(\bc_{t})^{-1}\bq_t
    \end{split}
\end{equation*}
\begin{equation*}
    \begin{split}
     \sum_{t=1}^{T}\bq_t^\top diag(\bc_{t})^{-1}\bq_t&\leq2\sum_{i=1}^{d}||\bq_{1:T,i}||_2\\
    \end{split}
\end{equation*}
\begin{equation*}
    \begin{split}
     \sum_{t=1}^{T}||\bq_t||^2_{\psi^*_{t-1}}&\leq  \eta\sum_{i=1}^{d}||\bq_{1:T,i}||_2
    \end{split}
\end{equation*}
So,
\begin{equation*}
    \begin{split}
        &\quad\sum_{t=1}^T[\bq_t^\top \bx_t+\phi(\bx_{t+1})]+ \frac{1}{\eta}\psi_T(\bx^*)+\psi^*_T(-\bz_T)\\
        &\leq \frac{1}{\eta}\psi_T(\bx^*)+\frac{\eta^2}{2}\sum_{i=1}^{d}||\bq_{1:T,i}||_2
    \end{split}
\end{equation*}
We also have 
\begin{equation*}
    \begin{split}
        \psi_T(\bx^*)&=\delta||\bx^*||_2^2+<\bx^*,diag(\bc_T)\bx^*>\\
        &\leq \delta||\bx^*||_2^2+||\bx^*||^2_\infty\sum_{i=1}^d||\bq_{1:T,i}||_2
    \end{split}
\end{equation*}
If $\mathbb{E}[\bq_t]=f'_t(\bx_t)$, take expectation for the quantization operation and set $G_\infty=\max_{t\leq T,i\leq d}||\bq_{1:T,i}||_2$

\begin{equation*}
    \begin{split}
        &\quad\sum_{t=1}^T\mathbb{E}_\bq[f_t(\bx_t)+\phi(\bx_t)-f_t(\bx^*)-\phi(\bx^*)]\\
        &\leq \sum_{t=1}^T\mathbb{E}_\bq[(f'_t(\bx_t)-\bq_t^\top(\bx_t-\bx^*)]\\
        &\quad+\mathbb{E}_\bq\left\{\sum_{t=1}^T[\bq_t^\top(\bx_t-\bx^*)+\phi(\bx_{t+1})-\phi(\bx^*)]\right\}\\
        &\leq \mathbb{E}_\bq\left\{\sum_{t=1}^T[\bq_t^\top \bx_t+\phi(\bx_{t+1})]+ \frac{1}{\eta}\psi_T(\bx^*)+\psi^*_T(-\bz_T)\right\}\\
        &\leq \frac{\delta}{\eta}||\bx^*||_2^2+\frac{1}{\eta}||\bx^*||^2_\infty\sum_{i=1}^d\mathbb{E}_\bq||\bq_{1:T,i}||_2 +\frac{\eta^2}{2}\sum_{i=1}^{d}\mathbb{E}_\bq||\bq_{1:T,i}||_2\\
        &=\frac{\delta}{\eta}||\bx^*||_2^2+(\frac{1}{\eta}||\bx^*||^2_\infty+\frac{\eta^2}{2}) dG_\infty\sqrt{T} 
    \end{split}
\end{equation*}

Finally,
\begin{equation*}
    \begin{split}
        &\quad\frac{1}{T}\sum_{t=1}^T\mathbb{E}_\bq[f_t(\bx_t)+\phi(\bx_t)-f_t(\bx^*)-\phi(\bx^*)]\\
        &\leq \frac{\delta||\bx^*||_2^2}{\eta T}+\frac{(\frac{1}{\eta}||\bx^*||^2_\infty+\frac{\eta^2}{2}) dG_\infty }{\sqrt{T}}
    \end{split}
\end{equation*}
\end{proof}

}
\end{document}